\documentclass[letterpaper]{article} 
\usepackage{aaai2027}  
\usepackage[hyphens]{url}  
\usepackage{graphicx} 
\usepackage{natbib}  
\usepackage{caption} 
\usepackage{algorithm}
\usepackage{algorithmic}

\usepackage{newfloat}
\usepackage{listings}
\DeclareCaptionStyle{ruled}{labelfont=normalfont,labelsep=colon,strut=off} 
\floatstyle{ruled}
\newfloat{listing}{tb}{lst}{}
\floatname{listing}{Listing}

\usepackage{booktabs}
\usepackage{amsmath}
\usepackage{amssymb}
\usepackage{amsthm}

\newtheorem{theorem}{Theorem}
\newtheorem{corollary}{Corollary}

\title{RCShift: Certifying When Partial Linkage Suffices for\\
Finite-Sample Decisions}

\author{
Shuheng Cao$^{1,*}$,
Ruiqi Chen$^{2,*}$,
Zhenhao Zhang$^{4,6,\dagger}$,
Renjie Cao$^{3,\dagger}$,
Siyu Zhang$^{1,\dagger}$,
Lingwei Dang$^{5,\dagger}$,
Jiajun Zhang$^{6,\dagger}$,
Tingting Dan$^{7,\ddagger}$
}

\affiliations{
\begin{tabular}{@{}c@{}}
$^{1}$University of California, San Diego
\quad
$^{2}$University of Michigan, Ann Arbor
\quad
$^{3}$Boston College
\\[0.2em]
$^{4}$ShanghaiTech University
\quad
$^{5}$South China University of Technology
\quad
$^{6}$Tsinghua University
\\[0.2em]
$^{7}$University of North Carolina at Chapel Hill
\\[0.45em]
{\small
$^{*}$Co-first authors
\qquad
$^{\dagger}$Equal contribution
\qquad
$^{\ddagger}$Corresponding author
}
\end{tabular}
}

\begin{document}

\maketitle

\begin{abstract}
Systems with costly gold outcomes and cheaper auxiliary observations must
decide how much record linkage to retain.  Complete pairing retains every joint
counter, while separate margins retain none.  Neither endpoint is calibrated
to a declared finite-sample decision.  Universal reconstruction can retain
cycle directions invisible to the likelihood-ratio family.  Family-exact
storage can exceed what the decision requires because certified residual loss
may fit within finite-sample slack.  We introduce RCShift, which
certifies two routes to sufficiency under a declared observation contract.  Its
exact mode characterizes minimum-cost family-exact storage through LR-visible
cycle directions.  Its approximate mode bounds reverse Le Cam deficiency.  Its
integer mode certifies whether a chosen set preserves the full experiment's
minimum integer record count at specified size and power.  In a rank-two
witness, one aligned counter preserves a four-record minimum.  An equal-cost
misaligned counter and the margins require eleven records, while universal
reconstruction requires two counters.  A local perturbation has positive
reverse deficiency yet retains the four-record minimum.  Proof-checked
scheduling bounds instantiate the contract before gold computation and yield
exact reconstruction on the admitted tree support.  RCShift turns
partial-linkage storage into decision-calibrated measurement design for the
declared family, costs, target, and common strictly positive support.
\end{abstract}


\section{Introduction}

Certification systems often pair expensive gold outcomes with cheaper
auxiliary observations.  The auxiliary stream may contain verifier decisions,
interval certificates, automated scores, or structured assays.  Both streams
concern the same records, but complete record-level linkage can be costly or
unavailable.  Identifiers may be removed, streams may arrive asynchronously,
or only selected joint aggregates may persist.  Complete paired tables and
separate margins are endpoints of a broader storage design space.  When do
selected joint counters already suffice for the declared finite-sample
decision?

This design problem is weaker than table reconstruction.  Margins define
transportation fibers, and cycles generate their unresolved table moves
\citep{sullivant2004gaps,slavkovic2014fibers}.  Universal reconstruction must
resolve every cycle.  A certification decision depends only on the declared
likelihood-ratio (LR) family's LR-visible cycle directions.  A retained counter
is valuable when it resolves one of these directions.
Counter value therefore depends on alignment with an LR-visible direction, not
on counter count alone.

Classical sufficiency and experiment comparison determine whether a given
statistic preserves an experiment
\citep{blackwell1953equivalent,lecam1964sufficiency,kay2000sufficiency}.
Task-aware representation learning asks which information is usable by a
predictive family and constructs representations relative to that family
\citep{xu2020usable,dubois2020decodable}.  Active feature acquisition chooses
costly measurements for predictive decisions, including nongreedy policies
for jointly informative features
\citep{li2021active,valancius2024nongreedy}.  These lines make information
value task dependent, but they do not separate family-exact storage from
finite-decision preservation after margins are fixed.  The unresolved design
problem is to choose costed joint counters and certify which sets preserve the
exact minimum integer record count at specified size and power.

We formalize this problem as \emph{certification-preserving linkage storage}
and introduce RCShift.  The observation contract fixes the support, released
margins, candidate counters, costs, declared laws, and decision target before
final gold outcomes are observed.  RCShift answers the central question through
two reductions.  Family-exact storage retains the LR-visible cycle directions.
When exactness fails, finite-decision preservation can still be certified if
residual loss fits within full-experiment decision slack.  Three certificate
modes implement these reductions.  The exact mode characterizes minimum-cost
family-exact storage.  The approximate mode bounds reverse deficiency.  The
integer mode applies the slack test.  A passing set preserves the full
experiment's minimum integer record count.  The storage requirements remain
nested.  Universal reconstruction implies family-exact storage, which implies
finite-decision preservation for the declared target.  Full LR cycle rank
returns the family-exact storage budget to universal reconstruction under unit
counter costs.

Controlled evidence tests both reductions.  Under unit costs, changing only
counter alignment gives
\begin{equation}
\begin{aligned}
&b_{\rm preserve}=b_{\rm task}=1<2=b_{\rm universal},\\
&K_{\rm full}=K_{\rm aligned}=4<11=K_{\rm misaligned}=K_{\rm margins}.
\end{aligned}
\end{equation}
The aligned counter preserves four records, while the misaligned counter and
margins require eleven.  A local perturbation retains the four-record minimum
despite positive reverse deficiency.  A full-rank control identifies
saturation.  Proof-checked scheduling bounds instantiate the observation
contract.

The paper makes three contributions.
\begin{enumerate}
\item \textbf{Direction-aware linkage storage.}  We formulate partial-linkage
retention as costed measurement design and characterize minimum-cost
family-exact storage through LR-visible cycle directions.
\item \textbf{Finite-sample loss certification.}  RCShift bounds reverse
deficiency and certifies when residual loss preserves the full experiment's
minimum integer record count under the stated transfer condition.
\item \textbf{Strict separations and proof-checked instantiation.}  A
seven-edge witness establishes controlled direction and threshold separations
with a full-rank boundary.  Proof-checked scheduling bounds instantiate the
observation contract and yield exact tree reconstruction.
\end{enumerate}

\section{Related Work}

\paragraph{Contingency geometry and data linkage.}
Fixed-margin contingency methods characterize transportation fibers and the
moves that preserve sufficient statistics
\citep{diaconis1998algebraic,sullivant2004gaps,slavkovic2014fibers}.
Record-linkage methods instead infer identity correspondences across sources
through probabilistic rules, graphs, or embeddings
\citep{fellegi1969theory,albakri2015linkage,trisedya2019entity,
li2020grapher,liu2022matching}.  RCShift takes fixed margins and candidate
counters as the observation contract, then selects aggregates that preserve a
declared decision family.  This design object differs from identity recovery
and universal table reconstruction.

\paragraph{Task-aware sufficiency and representation.}
Classical sufficiency uses factorization, likelihood ratios, and linear
summaries to identify information preserved for a statistical family
\citep{kay2000sufficiency,hara2012hierarchical}.  Blackwell comparison and Le
Cam deficiency extend this question to exact and approximate experiment
simulation
\citep{blackwell1953equivalent,lecam1964sufficiency}.  Information bottleneck
methods and neural mutual-information objectives construct compressed
predictive representations
\citep{alemi2017deep,wu2020graph,belghazi2018mine,hjelm2019learning,
poole2019variational,tschannen2020mutual}.  Predictive $\mathcal V$-information
and decodable bottlenecks make information value relative to a predictive family
\citep{xu2020usable,dubois2020decodable}.  RCShift instead selects costed
joint counters conditional on released margins.  It connects
family-relative information to exact finite record requirements at specified
size and power.

\paragraph{Cost-aware information acquisition.}
Feature selection and arbitrary-conditioning models represent compact or
partially observed feature sets
\citep{balin2019concrete,li2020acflow}.  Active acquisition chooses measurements
under costs and partial observations through information gain or
predictive utility
\citep{ma2019eddi,gong2019icebreaker,li2021active}.  Nongreedy acquisition also
captures feature sets whose value is joint rather than separable
\citep{valancius2024nongreedy}.  RCShift fixes cohort-level joint counters before
final gold outcomes, conditional on released margins.  Balanced gain-graph
deletion gives the exact scalar reduction
\citep{bowlin2012frustration}, and the row-span condition verifies every returned
set.  Strict counter complementarity also shows why singleton scores can miss
an informative pair.  RCShift therefore returns storage certificates for a
fixed cohort rather than sequential acquisition policies.

\paragraph{Learning and certifying scheduling.}
Learning-based combinatorial optimization develops learned search, routing, and
scheduling policies
\citep{khalil2017learning,li2018combinatorial,gasse2019exact,kwon2020pomo,
zhang2020dispatch,wang2025memory}.  Certified combinatorial optimization instead
attaches independently checkable proofs to solver claims
\citep{gocht2021parity,bogaerts2022certified,vandesande2026branch}.  RCShift uses
feasible schedules and analytic lower bounds to form checked auxiliary
observations before gold computation.  These observations instantiate the
declared contract for the RCShift analysis.

Together, these lines provide fixed-margin geometry, family-relative
information, cost-aware acquisition, and checkable optimization outputs.  They
neither determine which costed joint counters preserve a declared finite-sample
threshold after margins are fixed nor separate family-exact storage from
finite-decision preservation.

\section{Problem Formulation and the RCShift Compiler}
\label{sec:cpls}

A partial-linkage design is sufficient for the declared finite-sample decision
when its storage experiment attains the full experiment's minimum integer
record count at the specified size and power.  We now formalize the observation
contract and this target.

\paragraph{Observation model.}
Let $G=(\mathcal Y\cup\mathcal F,E)$ be a finite bipartite support graph and
let $\{P_\theta:\theta\in\Theta_0\cup\Theta_1\}$ be strictly positive laws on
$E$.  From $t$ i.i.d. records, the full experiment $\mathcal E_t$ observes the
joint count vector $N\in\mathbb Z_+^E$.  For $S\subseteq E$, the storage
experiment $\mathcal E_t^S$ observes all row margins, all column margins, and
the selected joint counts $(N_e)_{e\in S}$.  We denote its integer measurement
matrix by $A_S$.  All linear spaces are over $\mathbb R$.  Throughout,
$t\in\mathbb N_{\geq1}$ and $0<\alpha<\beta<1$.  Randomized tests use
independent external randomness.  The common strictly positive support
excludes parameter-dependent structural zeros and boundary laws.

For a randomized test $\phi$, define the maximin power envelope at null size
\(\alpha\)
\begin{equation}
 \Pi_{\mathcal E,t}(\alpha)
 :=\sup_{\phi:\,\sup_{\theta\in\Theta_0}P_\theta^t\phi\leq\alpha}
       \inf_{\theta\in\Theta_1}P_\theta^t\phi,
 \label{eq:maximin-envelope}
\end{equation}
and let
\begin{equation}
 K_{\mathcal E}(\alpha,\beta)
 :=\min\{t\in\mathbb N_{\geq1}\mid
 \Pi_{\mathcal E,t}(\alpha)\geq\beta\}.
 \label{eq:integer-minimum}
\end{equation}
The minimum of an empty set is $+\infty$.  Since $\mathcal E_t^S$ is a
deterministic coarsening of $\mathcal E_t$,
$K_{\mathcal E^S}(\alpha,\beta)\geq K_{\mathcal E}(\alpha,\beta)$.
Given nonnegative counter costs $c_e$, the target design quantity is
\begin{equation}
 b_{\rm preserve}
 :=\min_S\left\{\sum_{e\in S}c_e\ \middle|\
 K_{\mathcal E^S}(\alpha,\beta)=K_{\mathcal E}(\alpha,\beta)\right\}.
 \label{eq:b-preserve}
\end{equation}

Also define the universal-reconstruction budget
\begin{equation}
 b_{\rm universal}:=\min_{\substack{S:\ A_S\text{ injective}\\
 \text{on every margin fiber}}}\ \sum_{e\in S}c_e.
 \label{eq:b-universal}
\end{equation}

\paragraph{Direction reduction.}
Released margins leave ambiguity along the support graph's cycle space.
Universal reconstruction must remove every residual cycle.  Certification
only needs to resolve the declared family's LR-visible cycle directions.
Each retained counter restricts one coordinate of this ambiguity.  RCShift
therefore treats linkage storage as decision-calibrated measurement design rather
than full table recovery.

\paragraph{Threshold reduction.}
Record budgets are integers.  A continuous information loss may have no
decision cost when the full test has slack.  The same loss may add several
records near a power threshold.  We therefore separate reverse deficiency
from the discrete output \(K_{\mathcal E^S}(\alpha,\beta)\).  Counter cost and
gold-record cost are also distinct resources.  The design objective asks which
counters preserve a separately specified minimum integer record count.
Three certificate modes implement these reductions. The exact mode
characterizes minimum-cost family-exact storage. The approximate mode bounds
reverse deficiency. The integer mode applies the slack test. A passing set
preserves the full experiment's minimum integer record count.
Figure~\ref{fig:rcshift-overview} summarizes the declared observation
contract, the exact and approximate certification routes, and the controlled
finite-sample separations.

\begin{figure*}[t]
    \centering
    \includegraphics[width=\textwidth]{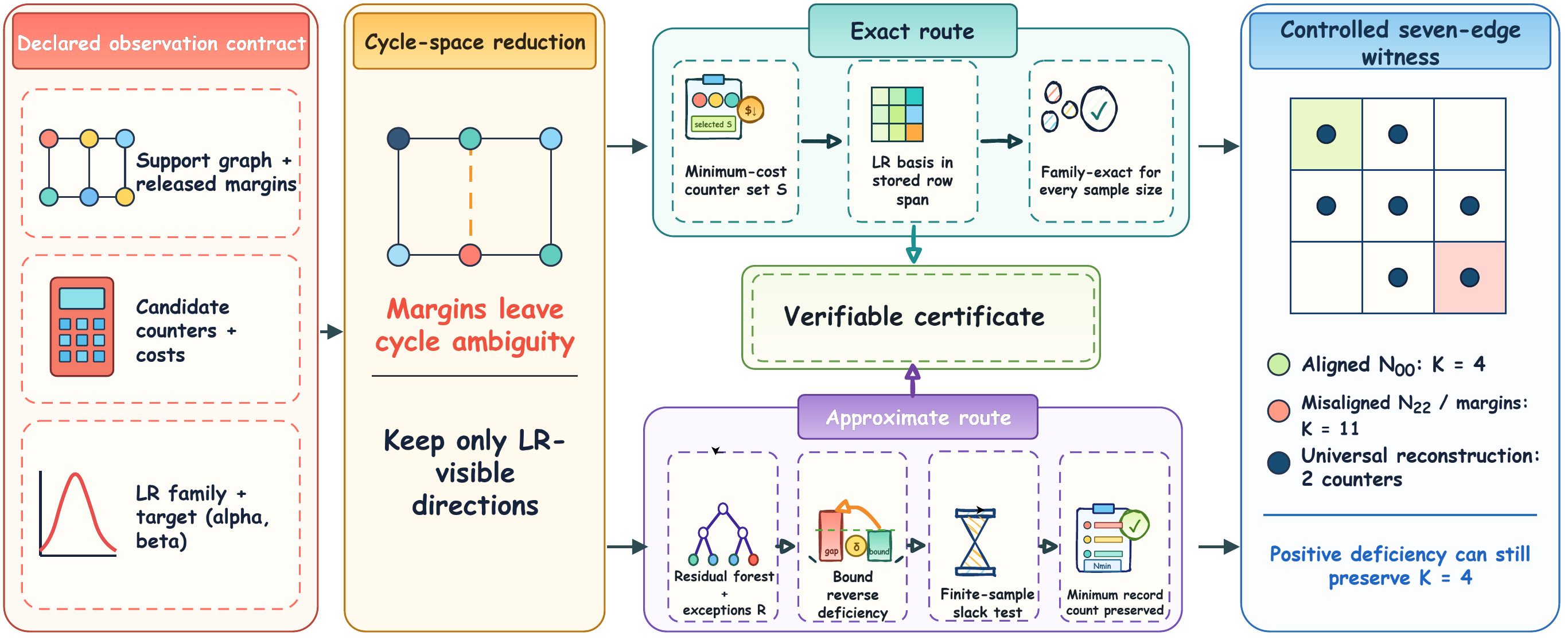}
    \caption{RCShift overview. }
    \label{fig:rcshift-overview}
\end{figure*}

Controlled evidence tests both reductions. Under unit costs, changing only
counter alignment gives
\begin{equation}
\begin{aligned}
&b_{\rm preserve}=b_{\rm task}=1<2=b_{\rm universal},\\
&K_{\rm full}=K_{\rm aligned}=4<11=K_{\rm misaligned}=K_{\rm margins}.
\end{aligned}
\end{equation}
Applications specify their retention costs through \(c_e\).

\paragraph{Information certificates.}
Fix $\theta_\star\in\Theta_0\cup\Theta_1$ and define the likelihood-ratio
subspace
\begin{equation}
 \mathcal L
 :=\operatorname{span}\left\{
 \left(\log\frac{P_\theta(e)}{P_{\theta_\star}(e)}\right)_{e\in E}:
 \theta\in\Theta_0\cup\Theta_1\right\}.
 \label{eq:lr-subspace}
\end{equation}
Let
\begin{equation}
 \epsilon_{S,t}
 :=\inf_T\sup_{\theta\in\Theta_0\cup\Theta_1}
 \left\|T P_{\theta,S}^t-P_\theta^t\right\|_{\rm TV}
 \label{eq:deficiency}
\end{equation}
be the reverse deficiency, where $T$ ranges over Markov kernels from
the stored experiment to the full experiment.

For $R\subseteq E\setminus S$ such that $G-(S\cup R)$ is a forest, set
\begin{equation}
 \begin{aligned}
 U_{S,t}(R)&:=\sup_\theta\{1-(1-P_\theta(R))^t\},\\
 U_{S,t}&:=\inf_R U_{S,t}(R).
 \end{aligned}
 \label{eq:upper-certificate}
\end{equation}
Fix a sound certificate map $u=(u_{S,t})$ with
$\epsilon_{S,t}\leq u_{S,t}$.  Equation~\eqref{eq:upper-certificate} gives
the default choice $u_{S,t}=U_{S,t}$.
For a binary subexperiment $(\theta_\star,\theta)$ and $c\geq0$, let
$D_c(P\|Q)=\sum_x[P(x)-cQ(x)]_+$ and define its contraction gap
\begin{equation}
 \begin{aligned}
 \Delta_{\theta,c}(S,t)
 &:=D_c(P_\theta^t\|P_{\theta_\star}^t)
  -D_c(P_{\theta,S}^t\|P_{\theta_\star,S}^t),\\
 L_{S,t}&:=\sup_{\theta,c}\frac{\Delta_{\theta,c}(S,t)}{1+c}.
 \end{aligned}
 \label{eq:lower-certificate}
\end{equation}
Write $\mathcal C=\ker(A_\varnothing)$.  For each edge, let
$r_e\in\mathcal C^\ast$ be the coordinate functional $r_e(z)=z_e$.  Let
$\mathcal W\subseteq\mathcal C^\ast$ be the subspace formed by restricting
$\mathcal L$ to $\mathcal C$.  Put
 $\mathcal R_S=\operatorname{span}\{r_e:e\in S\}$.

\paragraph{Certificate modes.}
Storage requirements are nested from universal reconstruction to family-exact
storage and finite-decision preservation.  RCShift evaluates them through three
certificate modes.  The exact mode is formalized in Parts 1 and 2, which
characterize direction retention and its minimum cost.  The approximate mode in
Part 3 bounds residual information loss.  The integer mode in Part 4 transfers
a sound loss bound to the finite-sample target.  Part 5 identifies full-rank
saturation.

\begin{theorem}[Certification-preserving linkage storage]
\label{thm:unified-storage}
Under the definitions above, the following statements hold.
\begin{enumerate}
 \item \textbf{Exact branch.}
 If
 \begin{equation}
   \mathcal L\subseteq\operatorname{rowspan}(A_S).
   \label{eq:rowspan-condition}
 \end{equation}
 then $\mathcal E_t^S$ and $\mathcal E_t$ are Blackwell equivalent for every
 $t\geq1$.  Their maximin envelopes and integer minima agree.  If the
 containment fails, equivalence fails for at least one finite $t$.

 \item \textbf{Exact dual compiler.}
 Counter sets yielding family-exact storage are characterized by
 \begin{equation}
   \mathcal W\subseteq\mathcal R_S,
   \qquad
   b_{\rm task}=\min_{S:\,\mathcal W\subseteq\mathcal R_S}
   \sum_{e\in S}c_e.
   \label{eq:dual-compiler}
 \end{equation}
 Counter value is nonseparable.  The monotone progress function
 $q(S)=\dim(\mathcal W\cap\operatorname{span}\{r_e:e\in S\})$ can have
 strict counter complementarity and need not be submodular.

 \item \textbf{Approximate branch.}
 Conditional on a supplied finite LR basis and evaluable family suprema, every
 $S$ satisfies the certificate sandwich
 \begin{equation}
   L_{S,t}\leq\epsilon_{S,t}\leq U_{S,t}.
   \label{eq:deficiency-sandwich}
 \end{equation}

 \item \textbf{Integer compilation.}
 Assume $t_\star=K_{\mathcal E}(\alpha,\beta)<\infty$.  Suppose a
 full-experiment test at $t_\star$ has size at most $a_0$ and worst-case
 alternative power at least $b_0$.  Write
 \(w=(\phi_\star,a_0,b_0,t_\star)\).  Set $\gamma_S(w,u)=1$ when
 $a_0+u_{S,t_\star}=0$.  Otherwise set
 $\gamma_S(w,u)=\min\{1,\alpha/(a_0+u_{S,t_\star})\}$.  If
 \begin{equation}
   \gamma_S(w,u)(b_0-u_{S,t_\star})\geq\beta,
   \label{eq:integer-transfer}
 \end{equation}
 then
 $K_{\mathcal E^S}(\alpha,\beta)=t_\star$.

 \item \textbf{Saturation.}
 Let $\ell=\dim\mathcal C$.  Under unit counter costs, if the restrictions of
 $\mathcal L$ span $\mathcal C^\ast$, then every storage set satisfying
 \eqref{eq:rowspan-condition} has size at least $\ell$.  A feedback-edge set
 of size $\ell$ attains equality.
\end{enumerate}
\end{theorem}

\begin{proof}
For any $\theta$, the full count likelihood ratio relative to $\theta_\star$
is
\begin{equation}
 \frac{P_\theta^t(N)}{P_{\theta_\star}^t(N)}
 =\exp\!\left(
 \left\langle
 \log(P_\theta/P_{\theta_\star}),N
 \right\rangle\right),
 \label{eq:count-lr}
\end{equation}
because the multinomial coefficient cancels between the normalized count
laws.  This ratio is constant on every $A_S$-fiber
for every $\theta$ exactly when every vector in $\mathcal L$ annihilates
$\ker(A_S)$, equivalently when \eqref{eq:rowspan-condition} holds.  The
factorization criterion gives sufficiency.  The stored statistic is a
deterministic function of the full table, which proves Blackwell equivalence.
Conversely, because $A_S$ is integral, its real kernel has a rational basis.
Clearing denominators supplies integer kernel moves.  An integer vector in
$\ker(A_S)$ whose pairing with an LR vector is
nonzero has positive and negative parts with equal total count.  Strict
positivity makes these an attainable same-statistic pair at some $t$ with
different likelihood ratios.  Thus universal equivalence fails.

Within $\mathcal C$, storing $S$ leaves
$\ker(A_S)=\bigcap_{e\in S}\ker(r_e)$.  The annihilator of this intersection
is $\operatorname{span}\{r_e:e\in S\}$, so the exact condition from Part 1 is
equivalent to \eqref{eq:dual-compiler}.  Non-submodularity already occurs in
the seven-edge graph of Corollary~\ref{cor:strict-storage}.  In the cycle basis
$c_1=(1,-1,-1,1,0,0,0)$ and
$c_2=(0,0,0,1,-1,-1,1)$, its LR target is
$\mathcal W=\operatorname{span}\{(1,0)\}$.  The coordinate functionals for
edges $(1,1)$ and $(2,2)$ are $(1,1)$ and $(0,1)$.  Hence $q$ is zero on each
singleton but one on their union.  This violates the submodular inequality.

We use $\|P-Q\|_{\rm TV}=\sup_A|P(A)-Q(A)|=\frac12\|P-Q\|_1$.
For the upper inequality in \eqref{eq:deficiency-sandwich}, define the
parameter-independent kernel $T_{S,R,t}$.  It subtracts stored cells from the
margins and applies leaf elimination on $G-(S\cup R)$.  An inconsistent input
is mapped to a fixed legal count vector with total count $t$.  Couple the full
sample to the event that no record falls in $R$.  On this event, leaf
elimination recovers the full table.  The failure probability under $\theta$
is $1-(1-P_\theta(R))^t$.  Optimizing over $R$ and taking the worst parameter
proves the upper bound.

For the lower inequality, suppose a kernel reconstructs every family member
within TV error $\eta$.  Hockey-stick divergence is data-processing monotone
and changes by at most $(1+c)\eta$ when both endpoints move by at most $\eta$.
Hence
\begin{equation}
 D_c(P_\theta^t\|P_{\theta_\star}^t)
 \leq D_c(P_{\theta,S}^t\|P_{\theta_\star,S}^t)+(1+c)\eta.
\end{equation}
Rearranging, then optimizing over $c,\theta$ and kernels, proves the lower
bound.

For \eqref{eq:integer-transfer}, a deficiency kernel emulates the full test
with null size at most $a_0+u_{S,t_\star}$ and worst-case power at least
 $b_0-u_{S,t_\star}$.  Independent randomized thinning by $\gamma_S(w,u)$ in
\eqref{eq:integer-transfer} restores size $\alpha$.  The displayed condition
therefore gives $K_{\mathcal E^S}\leq t_\star$.  Deterministic coarsening gives
the reverse inequality.

Finally, storing $s$ coordinates imposes at most $s$ independent constraints
on $\mathcal C$, so $s<\ell$ leaves a nonzero residual cycle.  Full dual rank
provides an LR vector that detects it, contradicting
\eqref{eq:rowspan-condition}.  A feedback-edge set leaves a forest and hence
achieves exactness with $\ell$ counters.
\end{proof}

\paragraph{Certified and exact decision budgets.}
Fix a full-test witness
$w=(\phi_\star,a_0,b_0,t_\star)$.  Define its certified budget under $u$ by
\begin{equation}
b_{\rm cert}(w,u):=
\min_S\left\{\sum_{e\in S}c_e\ \middle|\
\gamma_S(w,u)(b_0-u_{S,t_\star})\geq\beta\right\}.
\label{eq:b-cert}
\end{equation}
Theorem~\ref{thm:unified-storage} gives
$b_{\rm preserve}\leq b_{\rm cert}(w,u)$.  Here $b_{\rm task}$ is the minimum
cost of family-exact storage, while $b_{\rm cert}(w,u)$ is the minimum cost that
passes the fixed sufficient certificate.  Exact finite enumeration separately
identifies $b_{\rm preserve}$ in the seven-edge instance.

\section{Certifying the Two Reductions}

\paragraph{Exact graph form.}
For a two-point experiment with edge log-LR vector \(L\), a residual move \(z\)
is decision-invisible exactly when \(\langle L,z\rangle=0\).  The unstored gain
graph therefore needs to be balanced, not acyclic.  For an LR basis
\(B_1,\ldots,B_d\), exactness requires potentials
\(u_y^{(k)},v_f^{(k)}\) that satisfy
\begin{equation}
B_k(y,f)=u_y^{(k)}+v_f^{(k)}
\quad
\text{for every }(y,f)\in E\setminus S
\text{ and every }k.
\label{eq:simultaneous-balance}
\end{equation}
Minimum-cost family-exact storage is simultaneous weighted gain frustration,
whose scalar special case is NP-hard.  Universal reconstruction is weighted
feedback-edge deletion.  An ILP or branch-and-cut solver can propose \(S\).
Equation~\eqref{eq:rowspan-condition} verifies family-exact storage for the
returned set.

\paragraph{Counter complementarity.}
The dual form in Eq.~\eqref{eq:dual-compiler} identifies the design primitive.
Each retained counter contributes a coordinate functional toward the
LR-visible target subspace.  These contributions can be strict complements, so
a singleton ranking can miss a pair that spans a target direction.

\paragraph{Approximate decision control.}
For approximate storage, let $R$ contain low-mass edges whose removal, together
with stored coordinates, leaves a forest.  If no sampled record uses $R$, the
paired table is reconstructible.  This gives the upper certificate in
Eq.~\eqref{eq:upper-certificate}.  It is uniform over the declared family.
The hockey-stick contraction detects atoms that cross a likelihood-ratio
threshold after aggregation.  It certifies strict information loss.  Adjacent
NP crossings determine whether that loss changes the minimum integer record
count.

\subsection{Compiler Procedure}

The theorem induces a procedure that separates counter search from claim
verification.

\paragraph{Step 1. Contract and decision.}
Specify the structural support \(E\), retained margins, candidate counters,
costs, parameter families, and \((\alpha,\beta)\).  Structural zeros belong to
the observation contract rather than a finite pilot table.

\paragraph{Step 2. Exact family compilation.}
Exact mode accepts a finite LR basis \(B\) with rational or symbolic entries
and a certified zero test.  It checks each candidate \(S\) by exact
elimination using
\begin{equation}
\operatorname{rank}\begin{pmatrix}A_S\\B\end{pmatrix}
=\operatorname{rank}(A_S).
\label{eq:rank-verifier}
\end{equation}
A passing exact set yields family-exact storage for the declared family.
Equation~\eqref{eq:simultaneous-balance} gives the graph form.

\paragraph{Step 3. Approximate certification.}
For a set that fails exactness, search for exception sets $R$ such that
$G-(S\cup R)$ is a forest.  Evaluate the worst-family mass upper certificate
$U_{S,t}$ at the relevant record counts.  Separately compute or bound the
hockey-stick contraction $L_{S,t}$.  The pair gives a certified interval for
reverse deficiency.

\paragraph{Step 4. Integer decision compilation.}
Compute a full-test witness \(w=(\phi_\star,a_0,b_0,t_\star)\).  Apply
Eq.~\eqref{eq:integer-transfer}.  A passing set has
\(K_{\mathcal E^S}(\alpha,\beta)=t_\star\) because coarsening supplies the
reverse inequality.
For small finite experiments, exact reduced-experiment maximin computation or
NP enumeration can calculate $b_{\rm preserve}$ directly.

\paragraph{Step 5. Certificate verification.}
For a proposed set \(S\), the compiler returns
\(\mathsf C=(S,\mathsf{mode},B,\pi_u,\pi_L,w,\pi_{\rm opt})\).
The verifier checks Eq.~\eqref{eq:rank-verifier} in exact mode.  The upper
witness \(\pi_u=(R,m)\) verifies the forest condition, \(m\in[0,1]\),
\(m\geq\sup_\theta P_\theta(R)\), and
\(u_{S,t_\star}=1-(1-m)^{t_\star}\).  The optional witness
\(\pi_L=(t,\theta,c,\Delta)\) verifies a hockey-stick gap and certifies
\(\epsilon_{S,t}\geq\Delta/(1+c)\).
Approximate acceptance also checks \(w\) and
Eq.~\eqref{eq:integer-transfer}.  Exact acceptance proves equivalence for every
\(t\), while approximate acceptance proves
\(K_{\mathcal E^S}(\alpha,\beta)=t_\star\).  A claim of minimum cost
additionally requires \(\pi_{\rm opt}\) or exhaustive finite enumeration.

\section{Controlled Finite-Sample Separations}
\label{sec:strict-witness}

\paragraph{Direction separation at equal cost.}
The control holds counter cost fixed and changes only LR alignment.  Consider
the seven-edge support
\begin{equation}
E=\{00,01,10,11,12,21,22\},
\end{equation}
which has cycle rank two.  The null and alternative weights are
\begin{align}
W_0&=(8,16,16,8,8,1,8), & \sum W_0&=65,\\
W_1&=(16,8,8,16,16,1,8),& \sum W_1&=73.
\end{align}
The normalized likelihood ratio satisfies
\begin{equation}
\log\frac{P_1}{P_0}
=\log\frac{65}{73}\,\mathbf 1+(\log2)L,
\qquad
L=(1,-1,-1,1,1,0,0).
\end{equation}
The constant vector belongs to the margin row span.  The positive scale
\(\log2\) preserves cycle balance and likelihood-ratio ordering.  All stored
experiments retain both margins.  The aligned experiment also retains
$N_{00}$.  The same-cost misaligned experiment retains $N_{22}$.  The margins
experiment retains no joint counter.

After observing margins and $N_{00}$, the only residual cycle has alternating
contrast
\begin{equation}
L_{11}-L_{12}-L_{21}+L_{22}=1-1-0+0=0.
\end{equation}
Hence the aligned statistic is exactly sufficient.  Storing $N_{22}$ instead
leaves a first-cycle contrast
\begin{equation}
L_{00}-L_{01}-L_{10}+L_{11}=4,
\end{equation}
so equal counter cardinality does not imply equal decision value.

\paragraph{Exact finite direction certificate.}
For \(i\in\{0,1\}\) and \(z=A_Sn\), the exact atom probability of the stored
statistic is
\begin{equation}
p_{i,S}^{(t)}(z)=
\sum_{\substack{n\in\mathbb Z_+^E\\
\mathbf1^\top n=t,\ A_Sn=z}}
\frac{t!}{\prod_{e\in E}n_e!}
\prod_{e\in E}P_i(e)^{n_e}.
\label{eq:aggregated-atom}
\end{equation}
Every term is rational.  Group equal likelihood ratios into blocks
\(H_1,\ldots,H_m\) ordered from largest to smallest.  Write
\(h_{i,j}=\sum_{z\in H_j}p_{i,S}^{(t)}(z)\).  Set \(A_0=C_0=0\) and
\begin{equation}
\begin{aligned}
A_k&=\sum_{j\leq k}h_{0,j},&
C_k&=\sum_{j\leq k}h_{1,j},\\
k_\alpha&=\min\{k:A_k\geq\alpha\},&
\rho_\alpha&=\frac{\alpha-A_{k_\alpha-1}}
{A_{k_\alpha}-A_{k_\alpha-1}}.
\end{aligned}
\label{eq:np-boundary}
\end{equation}
The NP rule rejects the null on preceding blocks and rejects on
\(H_{k_\alpha}\) with probability \(\rho_\alpha\).
The exact Neyman--Pearson frontier is
\begin{equation}
\Pi_{\mathcal E^S,t}(\alpha)
=C_{k_\alpha-1}+\rho_\alpha(C_{k_\alpha}-C_{k_\alpha-1}).
\label{eq:np-frontier}
\end{equation}
The boundary denominator is positive under strict positivity.  For the full
and aligned experiments, the decisive powers are
\begin{equation}
p_3=\frac{109850}{389017}<.30,
\qquad
p_4=\frac{9542093}{28398241}>.30.
\label{eq:full-crossing}
\end{equation}

At \(\alpha=.05\) and \(\beta=.30\), exact enumeration gives
\(K_{\rm full}=K_{\rm aligned}=4<11=
K_{\rm misaligned}=K_{\rm margins}\).
Table~\ref{tab:exact} reports the certified adjacent frontiers.
\begin{table}[t]
\centering
\small
\begin{tabular}{lccc}
\toprule
Experiment & $K$ & $K-1$ upper & $K$ lower\\
\midrule
Full & 4 & .282378406 & .336010001\\
Aligned $N_{00}$ & 4 & .282378406 & .336010001\\
Misaligned $N_{22}$ & 11 & .299068452 & .310488554\\
Margins & 11 & .293856253 & .304398798\\
\bottomrule
\end{tabular}
\caption{Certified outward enclosures of exact rational NP frontiers.  Upper
entries are rounded upward and lower entries downward.}
\label{tab:exact}
\end{table}

Uniform deletion of one record defines a Markov kernel from every
size-\(t+1\) stored experiment to its size-\(t\) counterpart.  Each randomized
power envelope is therefore nondecreasing in \(t\).  The adjacent crossings
in Table~\ref{tab:exact} establish all four minima.  Boundary decisions use
independent external randomization.

The aligned counter preserves the full four-record minimum, whereas
the equal-cost \(N_{22}\) counter needs eleven records, matching margins.
Among singleton counters, exactly \(N_{00},N_{01},N_{10}\) yield family-exact
storage.  Thus \(b_{\rm preserve}=b_{\rm task}=1\), while universal
reconstruction requires two counters.

\paragraph{Misalignment loss certificate.}
For \(N_{22}\) at \(t=2\), exact aggregation with
\(c=4225/5329\) gives \(\Delta_c=384/5329\).  Therefore
\begin{equation}
\epsilon_{\{22\},2}\geq
\frac{\Delta_c}{1+c}
=\frac{192}{4777}.
\label{eq:strict-contraction}
\end{equation}

\paragraph{Threshold separation under positive deficiency.}
Positive information loss need not increase the integer minimum.  Endpoint
perturbations of at most \(\delta\) in one-record TV induce size-\(t\) product
perturbations of at most \(t\delta\).  Randomized thinning preserves the target
crossing when
\begin{equation}
\delta\leq
\frac{\alpha(p_t-\beta)}{t(\alpha+\beta)}.
\label{eq:tv-above}
\end{equation}
A one-fewer crossing with power \(p_t<\beta\) remains below the target when
\begin{equation}
\delta<
\frac{\alpha(\beta-p_t)}{t(\alpha+p_t)}.
\label{eq:tv-below}
\end{equation}
The smaller exact radius exceeds
\(1.3343345\times10^{-5}>10^{-5}\), with the displayed decimal lower bound
rounded downward.

An off-subspace perturbation within this neighborhood makes the aligned
counter insufficient.  Keep \(P_0\) fixed and define
\begin{equation}
P_{1,\eta}(e)=
\frac{P_1(e)\exp\{\eta\mathbf1[e=22]\}}
{\sum_{e'}P_1(e')\exp\{\eta\mathbf1[e'=22]\}}.
\label{eq:off-subspace-tilt}
\end{equation}
Write \(\epsilon_{S,t}(\eta)\) for the reverse deficiency of
\((P_0,P_{1,\eta})\).
Consider the cycle basis \(c_1=(1,-1,-1,1,0,0,0)\) and
\(c_2=(0,0,0,1,-1,-1,1)\).  The corresponding coordinate restrictions for the
\(00\) and \(22\) cells are \(r_{00}=(1,0)\) and \(r_{22}=(0,1)\).  The
perturbation adds \(\eta r_{22}\), which is not observed by \(N_{00}\).
The two size-two tables
\begin{equation}
n^+=\mathbf1_{11}+\mathbf1_{22},
\qquad
n^-=\mathbf1_{12}+\mathbf1_{21}
\end{equation}
have identical margins and \(N_{00}\), but their perturbed likelihood ratios
differ by \(\eta\).  Adding two \(00\) records preserves this fiber mismatch at
size four.  The unequal likelihood ratios violate factorization on the stored
fiber.  The finite kernel polytope is compact and the TV objective is
continuous.  Hence
\(\epsilon_{\{00\},4}(\eta)>0\) for every \(\eta\ne0\).

Let \(\delta_\eta=\|P_{1,\eta}-P_1\|_{\rm TV}\).  Continuity gives
\(\eta_0>0\) such that \(0<|\eta|\leq\eta_0\) implies
\(0<\delta_\eta\leq10^{-5}\).  Applying the unperturbed sufficiency kernel and
the product-TV bound gives
\begin{equation}
0<\epsilon_{\{00\},4}(\eta)\leq8\delta_\eta.
\label{eq:positive-deficiency}
\end{equation}
For \(0<\delta_\eta\leq10^{-5}\), the size-three full frontier remains below
\(.30\).  The size-four base rule retains null size \(\alpha\) and has
alternative power at least \(p_4-4\delta_\eta\).  At size four,
Theorem~\ref{thm:unified-storage}(4) gives
\begin{equation}
\frac{\alpha}{\alpha+8\delta_\eta}
\left(p_4-12\delta_\eta\right)>.30.
\end{equation}
Hence the perturbed full and stored experiments both have a minimum integer
record count of four even though their reverse deficiency is positive.

\paragraph{Family saturation.}
To identify the family boundary, fix a strictly positive reference law
\(P_\star\).  Given a gain vector \(v_j\), define
\begin{equation}
P_j(e)=
\frac{P_\star(e)\exp\{v_j(e)\}}
{\sum_{e'}P_\star(e')\exp\{v_j(e')\}}.
\end{equation}
Set \(\Theta_0=\{\star\}\) and \(\Theta_1=\{1,2\}\).  The constant
normalizers vanish on the cycle space.  The pair
\(v_1=L\) and \(v_2=(0,1,1,0,0,0,0)\) has cycle-projection rank one, so a singleton
counter set yields family-exact storage.  Replacing \(v_2\) by
\((0,0,0,0,1,0,0)\) gives full cycle-projection rank and raises the exact
family budget to two.

\begin{corollary}[Strict storage and decision separation]
\label{cor:strict-storage}
For the seven-edge rank-two experiment in
Section~\ref{sec:strict-witness}, at
$\alpha=1/20$ and $\beta=3/10$, under unit counter costs $c_e=1$ for every
$e\in E$,
\begin{equation}
 \begin{aligned}
 &b_{\rm preserve}=b_{\rm task}=1<2=b_{\rm universal},\\
 &K_{\rm full}=K_{N_{00}}=4<11=K_{N_{22}}=K_{\rm margins}.
 \end{aligned}
 \label{eq:strict-result}
\end{equation}
The exact singleton choices are $N_{00},N_{01},N_{10}$.  The counter $N_{22}$
has the contraction certificate
$\epsilon_{\{(2,2)\},2}\geq192/4777$.
The integer ordering persists when each endpoint moves by at most $10^{-5}$
in one-record TV.  The off-subspace family in
Eq.~\eqref{eq:off-subspace-tilt} has positive deficiency and preserves the
four-record minimum for every sufficiently small nonzero parameter.
\end{corollary}

\begin{proof}
The residual cycle after storing any of $N_{00},N_{01},N_{10}$ is orthogonal
to the LR vector, so Theorem~\ref{thm:unified-storage}(1) applies.  The exact
atom and frontier formulas in Eqs.~\eqref{eq:aggregated-atom}
and~\eqref{eq:np-frontier} give the adjacent crossings in
Table~\ref{tab:exact}.  Uniform random deletion makes every power envelope
nondecreasing in \(t\).  The crossings therefore prove the four integer
minima and \(b_{\rm preserve}=1\).  The cycle rank proves
\(b_{\rm universal}=2\).  Equation~\eqref{eq:strict-contraction} gives the
strict lower bound.  Equations~\eqref{eq:tv-above} and
\eqref{eq:tv-below} give the TV radius.  Equation~\eqref{eq:positive-deficiency}
proves the final claim.
\end{proof}

With nonuniform costs, the numerical identities $1<2$ need not hold.  The
family-exact storage objective becomes weighted gain-frustration, while universal
reconstruction becomes weighted feedback-edge deletion.  The cost-based
definitions, rather than cardinality, govern their values.

\section{Proof-Checked Contract Instantiation}

Independently checked optimization bounds instantiate the auxiliary side of the
observation contract.  For each flexible job-shop
scheduling instance \(x\), the fixed family \(\mathcal B\) orders operations by
operation index and baseline start, then applies deterministic earliest repair.
Let \(M_{\mathcal B}^\star(x)\) denote the family optimum over \(\mathcal B\).

The upper certificate supplies a complete schedule \(s_U\in\mathcal B\).  The
checker verifies family membership and feasibility under the declared machine,
precedence, downtime, and policy constraints.  It then recomputes
\({\rm UB}(x)=C_{\max}(s_U)\), which gives
\(M_{\mathcal B}^\star(x)\leq{\rm UB}(x)\).

The lower certificate records job-chain, workload, and release-precedence
bounds \(L_{\rm chain},L_{\rm load},L_{\rm rel}\).  The checker independently
recomputes these values and sets
\({\rm LB}(x)=\max\{L_{\rm chain},L_{\rm load},L_{\rm rel}\}\).
Job precedence, machine capacity, and public release constraints imply these
bounds for every schedule in \(\mathcal B\).  Every certificate pair that passes
both independent checkers therefore satisfies
\begin{equation}
{\rm LB}(x)\leq M_{\mathcal B}^\star(x)\leq{\rm UB}(x).
\label{eq:fjsp-bounds}
\end{equation}
For integer makespans, the fixed bins are
\(I_0=(-\infty,24]\), \(I_1=[25,36]\), and \(I_2=[37,\infty)\).
Define the gold label \(Y(x)=Y_j\) when
\(M_{\mathcal B}^\star(x)\in I_j\).  The checked interval then produces the
following set-valued auxiliary output
\begin{equation}
F(x)=\{Y_j\mid I_j\cap[{\rm LB}(x),{\rm UB}(x)]\ne\varnothing\}.
\label{eq:fjsp-contract}
\end{equation}
A contract is verified when it passes both checkers.  Its instance is admitted
when \(A=F(x)\) contains one class or two adjacent classes.  Each admitted
instance is encoded as \(F_A\), and Eq.~\eqref{eq:fjsp-bounds} guarantees
\(Y(x)\in A\).  The cohort uses \(F_1,F_2,F_{01},F_{12}\) for its two singleton
and two adjacent outputs.

Of 36 deterministic cohort instances, 34 passed both checkers.  Gold labels were
available for every verified contract, and all 34 contracts covered their
labels.  The admission rule retained 32 instances.

For the admitted cohort, \(E_{\rm contract}\) includes observed and
contract-allowed pairs, including zero-count cells.  Its six edges connect three
gold and four auxiliary vertices as a tree, so released margins reconstruct
every paired count.  The cohort supplies a proof-checked contract
instantiation with exact tree reconstruction.  The i.i.d. construction
separately proves the direction and threshold separations.

\section{Discussion and Limitations}

\paragraph{Decision-calibrated storage.}
RCShift certifies finite-sample sufficiency in two regimes under a declared
observation contract.  Family-exact storage retains every LR-visible cycle
direction and ignores reconstructive directions that the declared family cannot
see.  When family exactness fails, finite-decision preservation can still hold
if certified residual loss remains within full-experiment slack.  These regimes
separate information relevance from the decision cost of residual loss.

\paragraph{Two reductions beyond reconstruction.}
The controlled witnesses distinguish the two regimes.  At equal cost,
alignment changes the minimum from four records to eleven, which isolates
direction value.  A positive-deficiency perturbation preserves the four-record
minimum, which isolates threshold value.  Both reductions are distinct from
universal reconstruction.

\paragraph{Connection to AI information design.}
Fixed-margin geometry identifies unresolved cycle directions, while task-aware
objectives value these directions relative to a declared family
\citep{hara2012hierarchical,xu2020usable}. Cost-aware acquisition selects
measurements under limited budgets and partial observations
\citep{ma2019eddi,li2021active}. Certified optimization supplies independently
checkable auxiliary observations from solver outputs
\citep{bogaerts2022certified}.

\paragraph{Computation and scope.}
Minimum-cost search can be NP-hard and requires an optimality witness or
exhaustive enumeration.  Under unit costs, full LR cycle rank reaches the
universal budget and marks family saturation.  The guarantees assume a declared
LR family, additive costs, common strictly positive support, an exact LR basis,
and evaluable deficiency suprema.  Integer transfer also requires finite
\(t_\star\) and the stated slack condition.  Scheduling concerns fixed
\(\mathcal B\) and its deterministic cohort.

\section{Reproducibility Statement}

All reported finite-enumeration quantities use exact rational arithmetic.
The supplementary package contains aggregated atom probabilities,
randomized NP boundary rules, and executable checks for every reported
frontier and contraction value.

\section{Conclusion}

RCShift establishes decision-calibrated partial-linkage storage under a declared
observation contract.  It characterizes family-exact retention through
LR-visible cycle directions and certifies finite-decision preservation under a
sound loss transfer.  In the strict witness, one aligned counter preserves the
four-record minimum, while an equal-cost counter and the margins require eleven.
Universal
reconstruction requires two counters.  A local perturbation preserves the
four-record minimum despite positive reverse deficiency.  These results show
that counter value is directional and that positive information loss need not
incur finite-sample decision cost.  The claims hold for the declared family,
costs, target, and common strictly positive support.

\bibliography{refs}


\end{document}